\documentclass[letterpaper,10pt,journal,twoside]{IEEEtran}
\usepackage{cite}
\usepackage{amsmath,amssymb,amsfonts}
\usepackage{algorithm,algorithmic}
\usepackage{graphicx}
\usepackage{times}
\usepackage{float}
\usepackage{color}
\usepackage{xcolor}
\usepackage{amsthm}
\usepackage[caption=false,font=footnotesize]{subfig}
\usepackage{float}
\usepackage{siunitx}
\usepackage[prependcaption,textsize=tiny,color=blue!20!white]{todonotes}
\newtheorem{theorem}{Theorem}
\newtheorem{lemma}{Lemma}

\theoremstyle{definition}

\theoremstyle{remark}
\newtheorem{rmrk}{Remark}
\newtheorem{asmp}{Assumption}

\usepackage{dsfont}

\title{\huge Chance-Constrained Trajectory Optimization for Safe Exploration and Learning of Nonlinear Systems}

\author{Yashwanth Kumar Nakka, Anqi Liu, Guanya Shi, Anima Anandkumar, Yisong Yue, and Soon-Jo Chung\thanks{Manuscript received: May, 10, 2020; Revised: August, 11, 2020; Accepted: October, 1, 2020.}%Use only for final RAL version
         \thanks{This paper was recommended for publication by Editor Nancy Amato upon evaluation of the Associate Editor and Reviewers' comments.}
         \thanks{The authors are with California Institute of Technology.
         {{Email:}\{ynakka, anqiliu, gshi, anima, yyue, sjchung\}@caltech.edu}. This work was funded in part by the Jet Propulsion Laboratory, Caltech and the Raytheon Company. A. Liu is supported by a PIMCO Postdoctoral Fellowship. We acknowledge the contribution of  Irene S. Crowell in implementing Info-SNOC.}
         \thanks{Digital Object Identifier (DOI): see top of the page.}}
\begin{document}
\maketitle
% \thispagestyle{empty}
% \pagestyle{empty}
% Paper headers 
\markboth{IEEE Robotics and Automation Letters. Preprint Version. Accepted: October, 2020}{Nakka \MakeLowercase{\textit{et al.}}: Chance-Constrained Trajectory Optimization for Safe Exploration and Learning of Nonlinear Systems}% Use only for final RAL version
%%==========================Abstract===============================%%
\begin{abstract} 
Learning-based control algorithms require data collection with abundant supervision for training. Safe exploration algorithms ensure the safety of this data collection process even when only partial knowledge is available. We present a new approach for optimal motion planning with safe exploration that integrates chance-constrained stochastic optimal control with dynamics learning and feedback control. We derive an iterative convex optimization algorithm that solves an \underline{Info}rmation-cost \underline{S}tochastic \underline{N}onlinear \underline{O}ptimal \underline{C}ontrol problem (Info-SNOC). The optimization objective encodes control cost for performance and exploration cost for learning, and the safety is incorporated as distributionally robust chance constraints. The dynamics are predicted from a robust regression model that is learned from data. The Info-SNOC algorithm is used to compute a sub-optimal pool of safe motion plans that aid in exploration for learning unknown residual dynamics under safety constraints. A stable feedback controller is used to execute the motion plan and collect data for model learning. We prove the safety of rollout from our exploration method and reduction in uncertainty over epochs, thereby guaranteeing the consistency of our learning method. We validate the effectiveness of Info-SNOC by designing and implementing a pool of safe trajectories for a planar robot. We demonstrate that our approach has higher success rate in ensuring safety when compared to a deterministic trajectory optimization approach.
\end{abstract}
\begin{IEEEkeywords}Motion and Path Planning, Model Learning for Control, Machine Learning for Robot Control, Chance Constraints\end{IEEEkeywords}
%%=====================Introduction=======================%%
\IEEEpeerreviewmaketitle
\vspace{-5pt}
\section{Introduction}
\IEEEPARstart{R}obots deployed in the real world often operate in unknown or partially-known environments. Modeling the complex dynamic interactions with the environment requires high-fidelity techniques that are often computationally expensive. Machine-learning models can remedy this difficulty by approximating the dynamics from data~\cite{shi2019neural,dean2017sample,ostafew2016learning,punjani2015deep}. The learned models typically require off-line training with labeled data that are often not available or hard to collect in many applications. Safe exploration is an efficient approach to collect ground truth data by safely interacting with the environment. 

We present an episodic learning and control algorithm for safe exploration, as shown in Fig.~\ref{fig:work-flow}, that integrates learning, stochastic trajectory planning, and rollout for active and safe data collection. \emph{Rollout} is defined as executing the computed safe trajectory and policy using a stable feedback controller. The planning problem is formulated as an Information-cost Stochastic Nonlinear Optimal Control (Info-SNOC) problem that maximizes exploration and minimizes the control effort. Safety constraints are formulated as chance constraints. 

The propagation of uncertainty in the dynamic model and chance constraints in Info-SNOC are addressed by projecting the problem to the generalized polynomial chaos (gPC) space and computing a distributionally robust~\cite{calafiore2006distributionally,nakka2019nsoc} convex approximation. By building on~\cite{nakka2019nsoc}, we derive a sequential convex optimization solution to the Info-SNOC problem to plan a pool of sub-optimal safe and information-rich trajectories with the learned approximation of the dynamics. A sample of the trajectory pool is used as an input to the rollout stage to collect new data. To ensure real-time safety, the nonlinear feedback controller with a safety filter used in the rollout stage certifies bounded stochastic stability~\cite{dani2014observer}. The new data is used to learn an improved dynamic model. 

\begin{figure}
    \centering
    \includegraphics[height=2.3in,width=3.4in]{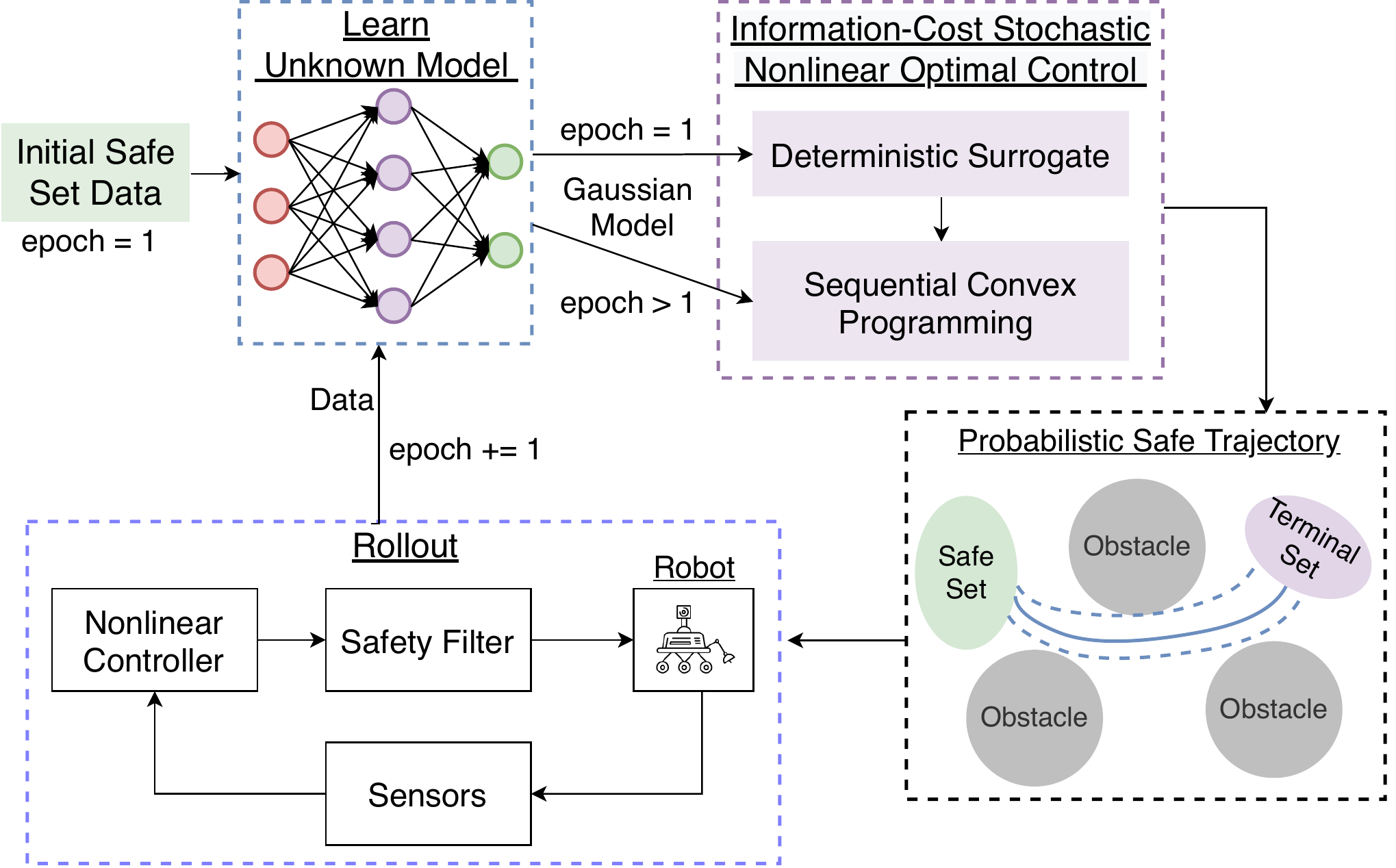}
    \caption{\textcolor{black}{An episodic framework for safe exploration using chance-constrained trajectory optimization. An initial estimate of the dynamics is computed using a known safe set~\cite{koller2018learning} and control policy. A probabilistic safe trajectory and policy that satisfies safety chance-constraints is computed using Info-SNOC for the estimated dynamics. This policy is used for rollout with a stable feedback controller to collect data.}\vspace{-18pt}} 
    \label{fig:work-flow}
\end{figure}

\subsubsection*{Related Work}
\label{sec:related_work}
Safe exploration for continuous dynamical systems has been studied using the following three frameworks: learning-based model-predictive control (MPC), dual-control, and active dynamics learning. Learning-based MPC~\cite{aswani2013provably,ostafew2016robust,hewing2018cautious,koller2018learning} has been studied extensively for controlling the learned system. These deterministic techniques are also applied for planning an information trajectory\footnote{An information trajectory is defined as a result of optimal motion planning that has more information about the unknown model compared to a fuel-optimal trajectory.} to learn online. The approach has limited analysis on safety of the motion plans that use recursive feasibility argument and by appending a known safe control policy. In contrast, stochastic trajectory planning~\cite{blackmore2011chance,nakka2019nsoc,mesbah2018stochastic,wang2020fast} naturally extends to incorporate probabilistic learned dynamic model. The safety constraints formulated as joint chance constraints~\cite{blackmore2011chance} facilitate a new approach to analyze the safety of the motion plans computed using Info-SNOC. The effect of uncertainty in the learned model on the propagation of dynamics is estimated using the method of generalized polynomial chaos (gPC)~\cite{nakka2019nsoc} expansion for propagation, which has asymptotic convergence to the original distribution, and provides guarantees on the constraint satisfaction.

Estimating unknown parameters while simultaneously optimizing for performance has been studied as a dual control problem~\cite{feldbaum1960dual}. Dual control is an optimal control problem formulation to compute a control policy that is optimized for performance and guaranteed parameter convergence. In some recent work~\cite{mesbah2018stochastic,cheng2015robust}, the convergence of the estimate is achieved by using the condition of persistency of excitation in the optimal control problem. Our method uses Sequential Convex Programming (SCP)~\cite{dinh2010local,morgan2014model,morgan2016swarm} to compute the persistent excitation trajectory. Recent work~\cite{buisson2019actively} uses nonlinear programming tools to solve optimal control problems with an upper-confidence bound~\cite{srinivas2010gaussian} cost for exploration without safety constraints. We follow a similar approach but formulate the planning problem as an SNOC with distributionally robust linear and quadratic chance constraints for safety. The distributionally robust chance constraints are convexified via projection to the gPC space. The algorithm proposed in this paper can be used in the MPC framework with appropriate terminal conditions for feasibility and to solve dual control problems with high efficiency using the interior point methods.

The contributions of the paper are as follows: a) we propose a new safe exploration and motion planning method by directly incorporating safety as chance constraints and ensuring stochastic nonlinear stability in the closed-loop control along with a safety filter; b) we derive a new solution method to the Info-SNOC  problem for safe and optimal motion planning, while ensuring the consistency and reduced uncertainty of our dynamics learning method; and c) we use a multivariate robust regression model~\cite{chen2016robust} under a covariate shift assumption to compute the multi-dimensional uncertainty estimates of the unknown dynamics used in Info-SNOC.
\subsubsection*{Organization} We discuss the preliminaries on robust regression learning method, the Info-SNOC problem formulation, and results on deterministic approximations of chance constraints in Sec.~\ref{sec:prelim}. The Info-SNOC algorithm along with the implementation of rollout policy is presented in Sec.~\ref{sec:computation}. In Sec.~\ref{sec:analysis}, we derive the end-to-end safety guarantees. In Sec.~\ref{sec:example}, we discuss a result of applying the Info-SNOC algorithm to the nonlinear three degree-of-freedom spacecraft robot model~\cite{nakka2018six}. We conclude the paper in Sec.~\ref{sec:conclusion}.
 %%=========================Problem Formulation=========================%%
% \vspace{-4pt}
\section{Preliminaries and Problem Formulation}
\label{sec:prelim}
\textcolor{black}{In this section, we discuss the preliminaries of the learning method used for modelling the unknown dynamics, the formulation of the Info-SNOC problem, and the deterministic surrogate projection using gPC.}
% \vspace{-7pt}
\subsection{Robust Regression For Learning}
An exploration step in active data collection for learning dynamics is regarded as a covariate shift problem. Covariate shift is a special case of distribution shift between training and testing data distributions. In particular, we aim to learn the unknown part of the dynamics $g(x,\bar{u})$ from state $x$ and control $\bar{u}$ defined in Sec.~\ref{sec:2B}. The covariate shift assumption indicates the conditional dynamics distribution given the states and controls remains the same while the input distribution of training ($\mathrm{Pr}_s(x,\bar{u})$) is different from the target input distribution ($\mathrm{Pr}_t(x,\bar{u})$). Robust regression is derived from a \emph{min--max} adversarial estimation framework, where the estimator minimizes a loss function and the adversary maximizes the loss under statistical constraints. The resulting Gaussian distributions provided by this learning framework are given below. For more technical details, we refer the readers to~\cite{chen2016robust,liu2019robust}. The output Gaussian distribution takes the form $\mathcal{N}(\mu_g, \Sigma_g)$ with mean $\mu_g$ and variance $\Sigma_g$:
\begin{align}
     \Sigma_g(x, \bar{u}, \theta_2)& = \left(2 \tfrac{\mathrm{Pr}_{s}(x,\bar{u})}{\mathrm{Pr}_{t}(x,\bar{u})} \theta_{2} + \Sigma_0^{-1}\right)^{-1}, \\
    \mu_g(x,\bar{u},\theta_1) &= \Sigma_g(x,\bar{u}, \theta_2)\left(-2\tfrac{\mathrm{Pr}_{s}(x,\bar{u})}{\mathrm{Pr}_{t}(x,\bar{u})} \theta_{1} \phi(x,\bar{u})  + \mu_0\Sigma_0^{-1}\right),\notag 
\end{align}
where $\mathcal{N}(\mu_0, \Sigma_0)$ is a non-informative (i.e., large $\Sigma_0$) base distribution, and $\phi(x,\bar{u})$ is the feature function that is learned using neural networks from data. The model parameters $\theta_1$ and $\theta_2$ are selected by maximizing the target log likelihood. The density ratio $\frac{\mathrm{Pr}_{s}(x,\bar{u})}{\mathrm{Pr}_{t}(x,\bar{u})}$ is estimated from data beforehand. Robust regression can handle multivariate outputs with correlation efficiently by incorporating neural networks and predicting a multivariate Gaussian distribution, whereas traditional methods like Gaussian process regression suffer from high-dimensions and require heavy tuning of kernels \cite{liu2019robust}.
\vspace{-7pt}
\subsection{Optimal and Safe Planning Problem}\label{sec:2B}
In this section, we present the finite-time chance-constrained stochastic optimal control problem formulation~\cite{nakka2019nsoc} used to design an information-rich trajectory. The optimization has control effort and terminal cost as performance objectives, and  the safety is modelled as a joint chance constraint of a set $\mathcal{F}$ defined by a polytope or a quadratic constraint. The full stochastic optimal control problem is as follows:\begin{align}
J^{*} =  & \underset{x(t),\bar{u}(t)}{\text{min}}
& & \scalebox{0.85}{$\mathbb{E} \left[\int_{t_{0}}^{t_{f}}J(x(t),\bar{u}(t))dt + J_f(x(t),\bar{u}(t))\right]$} \\
& \text{s.t.}
& & \scalebox{0.87}{$\dot{x}(t) = f(x(t),\bar{u}(t)) + \hat{g}(x(t),\bar{u}(t))$} \label{eq:approximate_dynamics}\\
&  & & \scalebox{0.9}{$\mathrm{Pr} (x(t) \in \mathcal{F}) \geq 1-\epsilon, \quad\forall t \in [t_{0},t_{f}]$} \\
&  & & \scalebox{0.9}{$\bar{u}(t) \in \mathcal{U} \quad \forall t \in [t_{0},t_{f}]$} \label{eq:control_bound}\\
&  & & \scalebox{0.9}{$x(t_{0}) = x_{0} \quad \mathbb{E}(x(t_f)) = \mu_{x_{f}}$},\label{eq:init_terminal}
\end{align}
where $x(t)\in \mathcal{X} \subseteq \mathbb{R}^n$ denotes the state of the dynamics, $x_{0}$ and $x_f$ are the initial and the terminal states respectively, the control $\bar{u} \in \mathcal{U} \subseteq \mathbb{R}^m$ is deterministic, $\hat{g}$ is the learned probabilistic model, and $\mathbb{E}$ is the expectation operator. The modelling assumptions and the problem formulation will be elaborated in the following sections.
\subsubsection{Dynamical Model}
The $\hat{g}$ term of~\eqref{eq:approximate_dynamics} is the estimated model of the unknown $g$ term of the original dynamics:
\begin{equation}
    \dot{\bar{p}} = f(\bar{p},\bar{u}) + \underbrace{g(\bar{p},\bar{u})}_\text{unknown},
    \label{eq:dynamics_original}
\end{equation}
where the state $\bar{p} \in \mathcal{X}$ is now considered deterministic, and the functions $f:\mathcal{X}\times \mathcal{U} \to \mathbb{R}^{n}$ and $g:\mathcal{X}\times\mathcal{U} \to \mathbb{R}^{n}$ are Lipschitz with respect to $\bar{p}$ and $\bar{u}$. The control set $\mathcal{U}$ is convex and compact.
\begin{rmrk}
The maximum entropy distribution with the known mean $\mu_g$ and covariance matrix $\Sigma_g$ of the random variable $\hat{g}$ is the Gaussian distribution $\mathcal{N}(\mu_g,\Sigma_g)$. 
\label{rmrk:gaussian_model}
\end{rmrk}
\noindent The learning algorithm computes the mean vector $\mu_{g}(\mu_x,\bar{u})$ and the covariance matrix $\Sigma_{g}(\mu_x,\bar{u})$ estimates of $g(x,\bar{u})$ that are functions of the mean $\mu_x$ of the state $x$ and control $\bar{u}$. 
Due to Remark~\ref{rmrk:gaussian_model}, the unknown bias term is modeled as a multivariate Gaussian distribution $\hat{g}(\mu_x,\bar{u}) \sim \mathcal{N}(\mu_g, \Sigma_g)$. The estimate $\hat{g}$ in (\ref{eq:approximate_dynamics}) can be expressed as
\begin{equation}\hat{g} = B(\mu_x,\bar{u}) \theta + \mu_g(\mu_x,\bar{u}), \label{eq:learned_model}\end{equation}where $\theta \sim \mathcal{N} (0,\mathrm{I})$ i.i.d and $B(\mu_x,\bar{u})B^\top(\mu_x,\bar{u}) = \Sigma_{g}(\mu_x,\bar{u})$. Using~\eqref{eq:learned_model}, (\ref{eq:dynamics_original}) can be written in standard Ito stochastic differential equation (SDE) form as given below, where $\theta dt = dW.$\begin{equation}
    dx = f(x,\bar{u}) dt + \mu_g(\mu_x,\bar{u}) dt + B(\mu_x,\bar{u}) dW
    \label{eq:dynamics_estimated}
\end{equation}
The existence and uniqueness of a solution to the SDE for a given initial distribution $x_0$ and control trajectory $\bar{u}$ such that $\mathrm{Pr}(|x(t_0)-x_0)| = 0) = 1$ with measure $\mathrm{Pr}$, is guaranteed by the \textit{Lipschitz condition} and \textit{Restriction on Growth condition}~\cite{arnold1974stochastic}.
The approximate system~(\ref{eq:dynamics_estimated}) is assumed to be controllable in the given feasible space.
\subsubsection{Cost Functional} The integrand cost functional $J= J_{\mathcal{C}} + J_{\mathcal{I}}$ includes two objectives: a) exploration cost $J_{\mathcal{I}}$ for achieving the maximum value of information for learning the unknown dynamics $g$, and b) performance cost $J_{\mathcal{C}}$ (e.g., minimizing the control effort). The integrand cost functional $J_{\mathcal{C}}$ for fuel optimality, which is convex in $\bar{u}$, is given as
\begin{equation}
    J_{\mathcal{C}} = \|\bar{u} \|_{s} \quad \text{where} \ \ s \in \{1,2,\infty\}
    \label{eq:control_cost}
\end{equation}
One example of $J_{\mathcal{I}}$  is the following variance-based information cost using each $i^{\mathrm{th}}$ diagonal element $\sigma^2_{i}$ in $\Sigma_g$. 
\begin{equation}
    J_{\mathcal{I}} = - \sum_{i=1}^{n}\sigma_{i}(\mu_x,\bar{u})    \label{eq:information_cost}
\end{equation}
The information cost $J_{\mathcal{I}}$ in~\eqref{eq:information_cost} is a functional of the mean $\mu_x$ of the state $x$ and control $\bar{u}$ at time $t$. By minimizing the cost $J_{\mathcal{I}}$, we maximize the information~\cite{srinivas2010gaussian} available in the trajectory $x$ to learn the unknown model $g$. The terminal cost functional $J_{f}$ is quadratic in the state $x$, $J_{f} = x(t_f)^\top Q_{f} x(t_f)$, where $Q_{f}$ is a positive semi-definite function.

\subsubsection{State and Safety Constraints} Safety is defined as a constraint on the state space $x$, $x(t) \in \mathcal{F}$ at time $t$. The safe set $\mathcal{F}$ is relaxed by formulating a joint chance constraint with risk of constraint violation as
\begin{equation}
    \mathrm{Pr}(x\in \mathcal{F}) \geq 1- \epsilon.
\end{equation}
The constant $\epsilon$ is called the risk measure of a chance constraint in this paper. We consider the polytopic constraint set $\mathcal{F}_\mathrm{lin} = \{x \in \mathcal{X}: \land_{i =1}^{k} a_{i}^\top x + b_{i} \leq 0\}$ with $k$ flat sides and a quadratic constraint set $\mathcal{F}_\mathrm{quad} = \{x \in \mathcal{X}: x^\top A x  \leq c\}$ for any realization $x$ of the state. The joint chance constraint $\mathrm{Pr}(\land_{i =1}^{k} a_{i}^\top x + b_{i} \leq 0) \geq 1- \epsilon$ can be transformed to the following individual chance constraint:
\begin{equation}
 \mathrm{Pr}(a_{i}^\top x + b_{i} \leq 0) \geq 1- \epsilon_{i},
 \label{eq:single_lin_cc}
\end{equation}
such that $\sum_{i = 1}^{k}\epsilon_{i} = \epsilon$. Here we use $\epsilon_{i} = \frac{\epsilon}{k}$. The individual risk measure $\epsilon_{i}$ can be optimally allocated between the $k$ constraints~\cite{ono2008iterative}. A quadratic chance constraint is given as:
\begin{equation}
    \mathrm{Pr}(x^\top A x \geq c) \leq \epsilon,
    \label{eq:quad_cc}
\end{equation}
where $A$ is a positive definite matrix. \textcolor{black}{We use a tractable conservative approximation of sets in~(\ref{eq:single_lin_cc},\ref{eq:quad_cc}) by formulating distributionally robust chance constraint~\cite{calafiore2006distributionally} for known mean $\mu_{x}$ and variance $\Sigma_{x}$ of the random variable $x$ at each time $t$.} 

\begin{lemma}
\label{lemma:lin_drcc}
The linear distributionally robust chance constraint $\inf_{x \sim (\mu_{x},\Sigma_{x})}\mathrm{Pr}(a^\top x + b \leq 0) \geq 1- \epsilon_{\ell}$ is equivalent to the deterministic constraint:
\begin{equation}
  a^\top \mu_{x} + b+ \sqrt{\frac{1-\epsilon_{\ell}}{\epsilon_{\ell}}}\sqrt{a^\top \Sigma_{x}a} \leq 0.
    \label{eq:lin_socc}
\end{equation}
\end{lemma}
\begin{proof}
 See Theorem 3.1 in~\cite{calafiore2006distributionally}.
\end{proof}
\begin{lemma}
\label{lemma:quad_drcc}
The semi-definite constraint on the variance $\Sigma_x$ \begin{equation}
     \mathrm{tr}(A\Sigma_x) \leq \epsilon_q c
     \label{eq:quad_sdp}
\end{equation}  
is a conservative deterministic approximation of the quadratic chance-constraint $\scalebox{0.95}{$\mathrm{Pr}( (x- \mu_{x})^\top A (x- \mu_{x}) \geq c)$} \leq \epsilon_{q}$.
\end{lemma}
\begin{proof}
See Proposition 1 in~\cite{nakka2019nsoc}. 
\end{proof}
\noindent The risk measures $\epsilon_{\ell}$ and $\epsilon_q$ are assumed to be given. For the rest of the paper, we consider the following individual chance-constrained problem,
\begin{equation}
\begin{aligned}
(x^{*},\bar{u}^{*})= & \underset{x(t),\bar{u}(t)}{\textrm{argmin}}\scalebox{0.95}{$\quad\mathbb{E}  \left[\int_{t_{0}}^{t_{f}} ((1-\rho)J_{\mathcal{C}} + \rho J_{\mathcal{I}})dt + J_f\right]$}, \\
\text{s.t.}& \{(\ref{eq:dynamics_estimated}),~(\ref{eq:lin_socc}),~(\ref{eq:quad_sdp}),~(\ref{eq:control_bound}),~(\ref{eq:init_terminal})\}
\end{aligned} \label{eq:stoptprob_sicc}
\end{equation}
that is assumed to have a feasible solution with $\rho \in [0,1]$. 
\vspace{-5pt}
\subsection{Generalized Polynomial Chaos (gPC)}
\label{subsec:gpc}
The problem~(\ref{eq:stoptprob_sicc}) is projected into a finite-dimensional space using the generalized polynomial chaos algorithm presented in~\cite{nakka2019nsoc}. We briefly review the ideas and equations. We refer the readers to \cite{nakka2019nsoc} for details on computing the functions and matrices that are mentioned below. In the gPC expansion, any $\mathcal{L}_{2}$-bounded random variable $x(t)$ is expressed as product of the basis functions matrix $\Phi(\theta)$ defined on the random variable $\theta$ and the deterministic time varying coefficients $X(t) = \begin{bmatrix}x_{10} &\cdots &x_{1\ell} & \cdots & x_{n0} & \cdots& x_{n\ell}\end{bmatrix}^\top$. The functions $\phi_{i}$, $i \in \{0,1,\dots,\ell\}$ are chosen to be Gauss-Hermite polynomials in this paper. Let 
$\Phi(\theta) = \begin{bmatrix} \phi_{0}(\theta) & \cdots & \phi_{\ell}(\theta)\end{bmatrix}^\top$. 
\begin{equation}
    x \approx \Bar{\Phi} X ;\ \text{where}\ \Bar{\Phi} = \mathbb{I}_{n\times n} \otimes \Phi(\theta)^\top
    \label{eq:state_kronck_notation}
\end{equation}
A Galerkin projection is used to transform the SDE~(\ref{eq:dynamics_estimated}) to the deterministic equation, \begin{equation}
    \dot{X} = \Bar{f}(X,\bar{u}) + \Bar{\mu}_{g}(\mu_x,\bar{u}) + \Bar{B}(\mu_x,\bar{u}).
    \label{eq:odefull_form}
\end{equation}
The exact form of the function $\bar{f}$ is given in~\cite{nakka2019nsoc}. The moments of the random variable $x$, $\mu_x$ and $\Sigma_x$ can be expressed as polynomial functions of elements of $X$. The polynomial functions defined in~\cite{nakka2019nsoc} are used to project the distributionally robust linear chance constraint in~(\ref{eq:lin_socc}) to a second-order cone constraint in $X$ as follows:
\begin{equation}
    (a^\top \otimes S)X + b + \sqrt{\frac{1-\epsilon_{\ell}}{\epsilon_{\ell}}}\sqrt{X^\top U N N^\top U^\top X} \leq 0,
    \label{eq:socc_gpc}
\end{equation}
where the matrices $S$, $U$, and $N$ are functions of the expectation of polynomials $\Phi$, that are given in~\cite{nakka2019nsoc}. The quadratic chance constraint in~(\ref{eq:quad_sdp}) is transformed to the following quadratic constraint in $X$:
\begin{equation}
    \sum_{i = 1}^{n} \sum_{k = 1}^{\ell} a_{i}\langle \phi_{k},\phi_{k}\rangle x_{ik}^2 \leq \epsilon_{q} c,
    \label{eq:sdp_gpc}
\end{equation}
where $A$ is a diagonal matrix with $a_{i}$ being the $i^\mathrm{th}$ diagonal element, and $\langle.,.\rangle$ denotes an inner product operation. Similarly, we project the initial and terminal state constraints $X(t_0)$ and $X(t_{f})$. In Sec.~\ref{sec:computation}, we present the Info-SNOC algorithm by using the projected optimal control problem.
\vspace{-2pt}
\section{Info-SNOC Main Algorithm}
\label{sec:computation}
We present the main algorithm for the architecture shown in Fig.~\ref{fig:work-flow} that integrates the learning method and the Info-SNOC method with rollout. We discuss an iterative solution method to Info-SNOC by projecting~\eqref{eq:stoptprob_sicc} to the gPC space. We formulate a deterministic optimal control problem in the gPC space and solve it using SCP~\cite{nakka2019nsoc,morgan2014model,morgan2016swarm} method. The gPC projection of~\eqref{eq:stoptprob_sicc} is given by the following equation:
\begin{equation}
\begin{aligned}
(X^{*},\bar{u}^{*})= & \underset{X(t),\bar{u}(t)}{\text{argmin}}\scalebox{0.9}{$\quad  \left[\int_{t_{0}}^{t_{f}} ((1-\rho)J_{\mathcal{C}} + \rho J_{\mathcal{I}})dt + \mathbb{E}{J}_{gPC_f}\right]$} \\
\text{s.t.}&~ \{(\ref{eq:odefull_form}),~(\ref{eq:socc_gpc}),~(\ref{eq:sdp_gpc}),~(\ref{eq:control_bound})\}\\
& X(t_0)= X_0,\quad X(t_f)= X_f.
\end{aligned} \label{eq:stoptprob_gpc}
\end{equation}
In SCP, the projected dynamic model~(\ref{eq:odefull_form}) is linearized about a feasible nominal trajectory and discretized to formulate a linear equality constraint. Note that the constraints~(\ref{eq:socc_gpc}) and~(\ref{eq:sdp_gpc}) are already convex in the states $X$. The terminal constraint $J_{f}$ is projected to the quadratic constraint $J_{gPC_f} = X_{f}^\top \Bar{\Phi}^\top Q_{f} \Bar{\Phi} X_{f}$. The information cost functional $J_{\mathcal{I}}$ from (\ref{eq:information_cost}) is expressed as a function of $X$ by using the polynomial representation~\cite{nakka2019nsoc} of $\mu_x$ in terms of $X$. Let $S= (X,\bar{u})^\top$ and the cost $J_{\mathcal{I}}$ is linearized around a feasible nominal trajectory $S^{o} = (X^{o},\bar{u}^{o})^\top$ to derive a linear convex cost functional $J_{d\mathcal{I}}$:
\begin{align}
   J_{d \mathcal{I}} = -\sum_{i= 1}^{n} \Big(\sigma_{i}(S^{o}) +\tfrac{\partial \sigma_{i}}{\partial S}\Big|_{S^{o}} (S- S^{o})\Big).\label{eq:linear_information_cost}
\end{align}
We use the convex approximation $J_{d\mathcal{I}}$ as the information cost in the SCP formulation of the optimal control problem in~\eqref{eq:stoptprob_gpc}. In the gPC space, we split the problem into two cases: a) $\rho = 0$ that computes a performance trajectory, and b) $\rho \in (0,1]$ that computes information trajectory to have stable iterations. The main algorithm is outlined below.

\subsubsection*{Algorithm}\textcolor{black}{We use an initial estimate of the model~\eqref{eq:learned_model} learned from data generated by a known safe control policy, and a \emph{nominal trajectory} $(x^{o},\bar{u}^{o})$ computed using deterministic SCP~\cite{morgan2014model} with nominal model to initialize Algorithm~\ref{Algo:inf_scp}. The stochastic model and the chance constraints are projected to the gPC state space, which is in line~2 of Algorithm~\ref{Algo:inf_scp}. The projected dynamics is linearized around the nominal trajectory and used as a constraint in the SCP. The projection step is only needed in the first epoch. The projected system can be directly used for $\text{epoch}>1$. The current estimated model is used to solve (\ref{eq:stoptprob_gpc}) using SCP, in line~7 with $\rho$ = 0, for a performance trajectory $(x_p,\bar{u}_p)$. The output $(x_p,\bar{u}_p)$ of this optimization is used as initialization to the Info-SNOC problem obtained by setting $\rho \in (0,1]$ to compute the information trajectory $(x_i,\bar{u}_i)$. The trajectory $(x_i,\bar{u}_i)$ is then sampled for a safe motion plan $(\bar{p}_d,\bar{u}_d)$ in line~9, that is used for rollout, in line~10, to collect more data for learning. The SCP step is performed in the gPC space $X$. After each SCP step, the gPC space coordinates $X$ are projected back to the random variable $x$ space. The Info-SNOC problem outputs a trajectory of random variable $x$ with finite variance at each epoch.}
\vspace{-3pt}
\begin{algorithm}
\caption{Info-SNOC using SCP~\cite{morgan2016swarm} and gPC~\cite{nakka2019nsoc}}
\label{Algo:inf_scp}
\begin{algorithmic}[1]
\STATE Initial Safe Set Data, Feasible Nominal Trajectory$(x^{o},\bar{u}^{o})$
\STATE gPC Projection as discussed in  Sec.~\ref{subsec:gpc}
\STATE Linearize the gPC cost and dynamics, see~\cite{nakka2019nsoc}
\STATE epoch = 1
\WHILE{Learning Criteria Not Satisfied}
    \STATE Learn $g\sim\mathcal{N}(\mu_g,\Sigma_g)$ using Robust Regression 
    \STATE $(x_{p},\bar{u}_{p})$ = SCP$\left((x^{o},\bar{u}^{o}), \rho =0\right)$, using~\eqref{eq:stoptprob_gpc}
    \STATE $(x_{i},\bar{u}_{i})$  = SCP$\left((x_{p},\bar{u}_{p}), \rho \in (0,1]\right)$, using~\eqref{eq:stoptprob_gpc}
    \STATE Sample $(x_{i},\bar{u}_{i})$ for $(\bar{p}_d,\bar{u}_{d})$
    \STATE Rollout using sample $(\bar{p}_{d},\bar{u}_{d})$ and $u_c$
    \STATE Data collection during rollout
    \STATE epoch $\gets$ epoch + 1
    \ENDWHILE
\end{algorithmic}
\end{algorithm}
\vspace{-6pt}
% Using the linear cost functional $J_{d\mathcal{I}}$ we compute a sub-optimal solution of~\eqref{eq:stoptprob_gpc}. The corresponding optimal value $J^*_{d\mathcal{I}}$ is an upper bound to the optimal cost $J^*_{\mathcal{I}} \leq J^*_{d\mathcal{I}}$.
\subsubsection*{Convergence and Optimality} The information trajectory $(x_i,\bar{u}_i)$ computed using SCP with the approximate linear information cost $J_{d\mathcal{I}}$~(\ref{eq:linear_information_cost}) is a sub-optimal solution of~\eqref{eq:stoptprob_gpc} with the optimal cost value $J^*_{d\mathcal{I}}$. Therefore, the optimal cost of~\eqref{eq:stoptprob_gpc} given by $J^{*}_{\mathcal{I}}$ is bounded above by $J^*_{d\mathcal{I}}$, $J^{*}_{\mathcal{I}} \leq J^*_{d\mathcal{I}}$. For the Info-SNOC algorithm, we cannot guarantee the convergence of SCP iterations to a Karush-Kuhn-Tucker point using the method in~\cite{morgan2016swarm,dinh2010local} due to the non-convexity of $J_{\mathcal{I}}$. Due to the non-convex cost function $J_{\mathcal{I}}$, the linear approximation $J_{d\mathcal{I}}$ of the cost $J_{\mathcal{I}}$ can potentially lead to numerical instability in SCP iterations. Finding an initial performance trajectory, $\rho = 0$, and then optimizing for information, $\rho \in (0,1]$, is observed to be numerically stable compared to directly optimizing for the original cost functional $J = J_{\mathcal{C}}+ J_{\mathcal{I}}$.

\subsubsection*{Feasibility} The initial phases of learning might lead to a large covariance $\Sigma_{g}$ due to the insufficient data, resulting in an infeasible optimal control problem. To overcome this, we use two strategies: 1) Explore the initial safe set till we find a feasible solution to the problem, and 2) Use slack variables on the terminal condition to approximately reach the goal accounting for a large variance. 

\subsection{Rollout Policy Implementation}
\label{subsec:exploration}The information trajectory $(x_i,\bar{u}_i)$ computed using the Info-SNOC algorithm is sampled for a pool of motion plans $(\bar{p}_d,\bar{u}_d)$. The trajectory pool is computed by randomly sampling the multivariate Gaussian distribution $\theta$ and transforming it using the gPC expansion $x(\theta) \approx \bar{\Phi}(\theta)X$. For any realization $\bar{\theta}$ of $\theta$, we get a deterministic trajectory $\bar{p}_{d} = \bar{\Phi}(\bar{\theta})X$ that is $\epsilon$ safe with respect to the distributionally robust chance constraints. The trajectory $(\bar{p}_d,\bar{u}_d)$ is executed using the closed-loop control law $u_{c} = u_{c}(\bar{p},\bar{p}_d,\bar{u}_d)$ for rollout, where $\bar{p}$ is the current state. To ensure real-time safety during the initial stages of exploration, a safe control policy $u_{s}$ is computed using the control barrier function-based safety filter. The properties of the control law $u_c$ and the safety during rollout are studied in the following section.

\vspace{-2pt}
\section{Analysis}
\label{sec:analysis} In this section, we present the main theoretical results analyzing the following two questions: 1) at any epoch $i$ how do learning errors translate to safety violation bounds during rollout, and 2) under what assumptions is the multivariate robust regression a consistent learning method as epoch $\to \infty$. The analysis proves that if the Info-SNOC algorithm computes a motion plan with finite variance, then learning is consistent and implies safety during rollout.  
\begin{asmp}
\label{asmp:sufficient_gpc}
The projected problem~(\ref{eq:stoptprob_gpc}) computes a feasible trajectory to the original problem~(\ref{eq:stoptprob_sicc}). The assumption is generally true if we choose a sufficient number of polynomials~\cite{nakka2019nsoc}, for the projection operation.
\end{asmp}
\begin{asmp}
\label{asmp:learning_bounds}
The probabilistic inequality $\mathrm{Pr}(\|g(\bar{p},\bar{u})-\mu_{g}(\bar{p},\bar{u})\|^2 _2 \leq c_1) \geq 1- \epsilon_{\mathrm{\ell b}}$ holds, where $\epsilon_{\mathrm{\ell b}}$ is small, for the same input $(\bar{p},\bar{u})$ to the original model $g$, and the mean $\mu_g$ of the learned model. Using this inequality, we can say that the following bounds hold with high probability:
\begin{equation}
    \begin{aligned}
     \|g(\bar{p},\bar{u})-\mu_{g}(\bar{p},\bar{u})\|^2 _2 \leq c_1, \mathrm{tr}\left(\Sigma_g \right) \leq c_2, \end{aligned}\label{eq:learning_bounds}
\end{equation}
\end{asmp}
\noindent where $c_2 = c_1 \epsilon_{\mathrm{\ell b}}$. As shown in~\cite{liu2019robust}, the mean predictions made by the model learned using robust regression is bounded by $c_1$, which depends on the choice of the function class for learning. The variance prediction $\Sigma_g$ is bounded by design. With Assumptions~\ref{asmp:sufficient_gpc} and~\ref{asmp:learning_bounds}, the analysis is decomposed into the following three subsections.
\vspace{-7pt}
\subsection{State Error Bounds During Rollout}
\label{subsec:state_bounds}
The following assumptions are made on the nominal system $\dot{\bar{p}}=f(\bar{p},u)$ to derive the state tracking error bound during rollout.
\begin{asmp}
\label{asmp:stability}
There exists a globally exponentially stable (i.e., finite-gain $\mathcal{L}_p$ stable) tracking control law $u_c = u_c(\bar{p},\bar{p}_{d},\bar{u}_{d})$ for the nominal dynamics $\dot{\bar{p}}= f(\bar{p},u_c)$. The control law $u_c$ satisfies the property $u_{c}(\bar{p}_d,\bar{p}_d,\bar{u}_d) = \bar{u}_d$ for any sampled trajectory $(\bar{p}_{d},\bar{u}_{d})$ from the information trajectory $(x_{i},\bar{u}_{i})$. At any time $t$ the state $\bar{p}$ satisfies the following inequality, when the closed-loop control $u_c$ is applied to the nominal dynamics, 
\begin{equation*}
\scalebox{0.95}{$M(\bar{p},t)\frac{\partial f}{\partial \bar{p}} + \left(\frac{\partial f}{\partial \bar{p}}\right)^\top M(\bar{p},t) + \frac{d}{dt}M(\bar{p},t) \leq - 2 \alpha M(\bar{p},t)$},
\end{equation*}
where $f = f(\bar{p},u_c(\bar{p},\bar{p}_d,\bar{u}_d))$, $\alpha>0$, $M(\bar{p},t)$ is a uniformly positive definite matrix with $(\lambda_{\min}(M)\|\bar{p}\|^2 \leq \bar{p}^\top M(\bar{p},t)\bar{p} \leq \lambda_{\max}(M)\|\bar{p}\|^2)$, and $\lambda_{\max}$ and $\lambda_{\min}$ are the maximum and minimum eigenvalues.
\end{asmp}
\begin{asmp}\label{asmp:bounded_g}
The unknown model $g$ satisfies the bound $\|\left(g\left(\bar{p},u_c\right) - g(\bar{p}_d,\bar{u}_d)\right)\|^2_2 \leq c_3$. 
\end{asmp}
\begin{asmp}\label{asmp:bounded_density}
The probability density ratio $\tfrac{\rho_{x_i(t)}}{\rho_{x_i(0)}} \leq r$ is bounded, where the functions $\rho_{x_i(0)}$ and $\rho_{x_i(t)}$ are the probability density functions of $x_i$ at time $t=0$ and $t$ respectively.
\end{asmp}
\begin{lemma}\label{lemma:bound}
Given that the estimated model~\eqref{eq:learned_model} satisfies the Assumption~\ref{asmp:learning_bounds}, and the systems~(\ref{eq:dynamics_original}) and~\eqref{eq:dynamics_estimated} satisfy Assumptions~\ref{asmp:stability},~\ref{asmp:bounded_g},~\ref{asmp:bounded_density}, if the control $u_{c} = u_{c}(\bar{p},x_i,\bar{u}_i)$ is applied to the system~\eqref{eq:dynamics_original}, then the following condition holds at time $t$
\begin{align}
    \mathbb{E}_{x_{i}(t)} (\|\bar{p}-x_i\|^2_{2}) & \leq  \tfrac{\lambda_{\max}(M)}{2\lambda_{\min}(M) \alpha_{m}} (c_1 + c_2 + c_3) r \label{eq:propagation_error}\\&
    +  \tfrac{\lambda_{\max}(M)r}{\lambda_{\min}(M)} \mathbb{E} \left( \|\bar{p}(0)-x_i(0)\|^2\right) e^{-2\alpha_{m}t}, \nonumber
\end{align}
where $(x_i,\bar{u}_i)$ is computed from~\eqref{eq:stoptprob_gpc} and $\alpha_{m} = (\alpha - 1)$. The states $\bar{p}\in \mathcal{X}$, and $x_i\in \mathcal{X}$ are feasible trajectories of the deterministic dynamics~(\ref{eq:dynamics_original}) and the SDE~(\ref{eq:dynamics_estimated}) for the initial conditions $\bar{p}(0) \in \mathcal{X}$ and $x_i(0)  \in \mathcal{X}$ respectively at $t\geq t_{0}$.
\end{lemma}
\begin{proof} See Lemma 2 in~\cite{dani2014observer}.\end{proof}
Lemma~\ref{lemma:bound} states that the expected mean squared error $\mathbb{E}(\|\bar{p}-x_i\|^2)$ is bounded by $\tfrac{\lambda_{\max}(M)(c_1 + c_2 + c_3)r}{2\alpha_{m}\lambda_{\min}(M)}$ as $t \to \infty$ when the control law $u_c$ is applied to the dynamics in~\eqref{eq:dynamics_original}. The bounded tracking performance leads to constraint violation, which is studied in the next section. 
\vspace{-7pt}
\subsection{Safety Bounds}
The safety of the original system~\eqref{eq:dynamics_original} for the linear and quadratic chance constraints during rollout with a controller $u_{c}$ discussed in Sec.~\ref{subsec:state_bounds} is analyzed in Theorems~\ref{thm:quad_chance_constraint} and~\ref{thm:lin_chance_constraint}.
\begin{theorem}
\label{thm:quad_chance_constraint} 
Given a feasible solution $(x,\bar{u}_{x})$ of (\ref{eq:stoptprob_sicc}), with the quadratic chance constraint $\mathrm{Pr}((x-\mu_{x})^\top A (x-\mu_{x}) \geq c) \leq \frac{\mathbb{E}((x-\mu_{x})^\top A (x-\mu_{x}))}{c} \leq \epsilon_q$, the trajectory $\bar{p}$ of the deterministic dynamics~\eqref{eq:dynamics_original} satisfies the following inequality at any time $t$:
\begin{align}
    &(\bar{p} - \mu_{x})^\top A (\bar{p}-\mu_{x}) \leq  \lambda_{\max}(A) \mathbb{E}_{x} \left( \|\bar{p} - x\|^2_{2}\right),  
    \label{eq:quad_violation}\end{align}
where $\mathbb{E} \left( \|\bar{p} - x\|^2_{2}\right)$ is bounded as defined in Lemma~\ref{lemma:bound}.
\end{theorem} \begin{proof}
Consider the expectation of the ellipsoidal set $(\bar{p} - \mu_{x})^\top A (\bar{p} - \mu_{x})$. Using $\bar{p} - \mu_{x} = \bar{p} - x + x- \mu_{x}$, the expectation of the set can be expressed as follows
\begin{align}
    &\mathbb{E}\left((\bar{p} - \mu_{x})^\top A (\bar{p} - \mu_{x}) \right) =  \mathbb{E}\left((\bar{p} - x)^\top A (\bar{p} - x)\right) \label{eq:thm1_eq1}\\
    + &\mathbb{E}\left((x - \mu_{x})^\top A (x - \mu_{x})\right) + 2\mathbb{E}\left((\bar{p} - x)^\top A (x - \mu_{x}) \right). \nonumber
\end{align}
Using the following equality: 
\begin{align*}
\mathbb{E}\left((\bar{p} - x)^\top A (x - \mu_{x}) \right) = -\mathbb{E}\left((x - \mu_{x})^\top A (x - \mu_{x}) \right),
\end{align*}
and $\mathbb{E}\left((x - \mu_{x})^\top A (x - \mu_{x}) \right) \geq 0$ in~\eqref{eq:thm1_eq1}, we obtain the constraint bound in~\eqref{eq:quad_violation}. Using the feedback tracking bound~(\ref{eq:propagation_error}) in~\eqref{eq:quad_violation}, we can show that the constraint violation bound is a function of learning bounds $c_1$, $c_2$, and $c_3$. \end{proof}
Note that if the learning method converges, i.e., $c_1\to0$, $c_2\to0$, and $c_3\to0$, then $\bar{p} \to \mu_{x}$. The quadratic constraint violation in~(\ref{eq:quad_violation}) depends on the tracking error and the size of the ellipsoidal set described by $A$.  
\begin{theorem} \label{thm:lin_chance_constraint}
Given a feasible solution $(x,\bar{u}_{x})$ of~(\ref{eq:stoptprob_sicc}) with $\inf_{x\sim (\mu_{x},\Sigma_{x})}\mathrm{Pr}(a^\top x + b \leq 0) \geq 1 - \epsilon_{\ell}$, the trajectory $\bar{p}$ of the deterministic dynamics~\eqref{eq:dynamics_original}, with control $u_c$, satisfies the following inequality at any time $t$:
\begin{equation}
\inf_{x\sim (\mu_{x},\Sigma_{x})}\mathrm{Pr}\left(a^\top \bar{p} + b \leq \delta_{\ell}(x) \right) \geq 1-\epsilon_{\ell}, \label{eq:lin_violation}
\end{equation} where $\delta_{\ell}(x) = \scalebox{0.95}{$\|a\|_{2} \mathbb{E}_{x}(\|(\bar{p} - x)\|_{2}) - \|a\|_{2} \sqrt{\frac{1-\epsilon_{\ell}}{\epsilon_{\ell}}}\sqrt{c_4} $}$, $\mathbb{E}_{x}(\|(\bar{p} - x\|_{2})$ is bounded as defined in~\eqref{eq:propagation_error} and $\mathrm{tr}(\Sigma_{x}) = c_4$.
\end{theorem}
\begin{proof} From Lemma~\ref{lemma:lin_drcc}, the feasible solution $(x,u_{x})$ satisfies the equivalent condition $\mathcal{P}(\mu_{x},\Sigma_{x}) \leq 0$, where $\mathcal{P}(\mu_{x},\Sigma_{x}) = a^\top \mu_{x} + b+ \sqrt{\frac{1-\epsilon_{\ell}}{\epsilon_{\ell}}}\sqrt{a^\top \Sigma_{x}a}$ for the risk measure $\epsilon_{\ell}$, mean $\mu_{x}$, and covariance $\Sigma_{x}$. Consider the similar condition for the actual trajectory $\bar{p}(t)$, $\mathcal{P}(\mu_{\bar{p}},\Sigma_{\bar{p}})$:
% \todo[inline]{You cannot start an equation of an undefined function H() below. If the first line defines (H) you should state that first.}
\begin{align}
& \mathcal{P}(\mu_{\bar{p}},\Sigma_{\bar{p}}) = a^\top \mu_{\bar{p}} + b+ \sqrt{\frac{1-\epsilon_{\ell}}{\epsilon_{\ell}}}\sqrt{a^\top \Sigma_{\bar{p}} a} \nonumber\\
& = a^\top \mu_{x}  + a^\top (\mu_{\bar{p}} - \mu_{x})+ b + \sqrt{\frac{1-\epsilon_{\ell}}{\epsilon_{\ell}}}\sqrt{a^\top \Sigma_{\bar{p}} a} \nonumber\\ 
& + \sqrt{\frac{1-\epsilon_{\ell}}{\epsilon_{\ell}}}\sqrt{a^\top \Sigma_{x} a} - \sqrt{\frac{1-\epsilon_{\ell}}{\epsilon_{\ell}}}\sqrt{a^\top \Sigma_{x} a} \label{eq:theorem2_eq1}
\end{align}
Note that since the system~\eqref{eq:dynamics_original} is deterministic, we have $\mu_{\bar{p}} = \bar{p}$ and $\Sigma_{\bar{p}} = 0$. Using $\mathcal{P}(\mu_{x},\Sigma_{x}) \leq 0$, the right hand side of the  above inequality reduces to the following:
\begin{align}
\mathcal{P}(\mu_{\bar{p}},\Sigma_{\bar{p}}) & \leq a^\top (\bar{p} - \mu_x) - \sqrt{\frac{1-\epsilon_\ell}{\epsilon_\ell}}\sqrt{a^\top \Sigma_{x} a}. \label{eq:theorem_eq2}
\end{align}
Using the decomposition $\Sigma_{x} = \tilde{G}^\top \tilde{G}$, Cauchy-Schwarz's inequality, and Jensen's inequality, we have   $a^\top (\mu_{\bar{p}} - \mu_{x}) \leq \|a\|_2 \|(\bar{p} - \mu_{x})\|_{2} \leq \|a\|_2 \mathbb{E}_{x}(\|(\bar{p} - x)\|_{2})$.
Using the sub-multiplicative property of $\ell_2$-norm  in the inequality above, we have $\mathcal{P}(\mu_{\bar{p}},\Sigma_{\bar{p}}) \leq \|a\|_{2} \left(\mathbb{E}(\|(\bar{p} - x)\|_{2}) -  \sqrt{\frac{1-\epsilon_{\ell}}{\epsilon_{\ell}}} \|\tilde{G}\|_{F} \right)$. Assuming that $\mathrm{tr}(\Sigma_{x}) = c_4$ in the above inequality, we have 
\begin{align*}
    \mathcal{P}(\mu_{\bar{p}},\Sigma_{\bar{p}}) \leq &\|a\|_{2} \mathbb{E}(\|(\bar{p} - x)\|_{2}) - \|a\|_{2} \sqrt{\frac{1-\epsilon_{\ell}}{\epsilon_{\ell}}} \sqrt{c_4}.
\end{align*}
The above inequality is equivalent to the probabilistic linear constraint in~\eqref{eq:lin_violation}. The bound $\delta_{\ell}(x) = \|a\|_{2} \mathbb{E}(\|(\bar{p} - x)\|_{2}) - \|a\|_{2}\sqrt{\frac{1-\epsilon_{\ell}}{\epsilon_{\ell}}}\sqrt{c_4}$ is a function of the learning bounds in~(\ref{eq:learning_bounds}) by substituting the feedback tracking bound in~\eqref{eq:propagation_error}.
\end{proof}
The linear constraint is offset by $\delta_{\ell}$ leading to constraint violation of the original formulation~\eqref{eq:single_lin_cc}. Note that, if $c_1\to0$, $c_2\to0$, $c_3\to0$, and $c_4 \to 0$, then $\delta_{\ell} \to 0$. In order to ensure real-time safety during trajectory tracking, we use a high gain control for disturbance attenuation with safety filter augmentation for constraint satisfaction.
\vspace{-4pt}
\subsection{Consistency}
Data is collected during the rollout of the nonlinear system to learn a new model for the next epoch. For epoch $e$, the predictor $\hat{g}^e$ follows a multivariate Gaussian distribution $\mathcal{N}(\mu_g^e, \Sigma_g^e)$ and $g$ is the empirical true data. We assume that set $\mathcal{X}_e \subset \mathcal{X}$ generated by the optimization problem in \eqref{eq:stoptprob_sicc} for the first $e$ iterations is a discretization of $\mathcal{X}$. Assuming that there exists a global optimal predictor $\hat{g}^*$ in the function class $\mathcal{G}$ that can achieve the best error $\epsilon^{*}$ at each epoch $e$: 
\begin{equation}
\max_{x \in \mathcal{X}_e} {\|g - \hat{g}^*\|}_2 = \epsilon^*,
\label{eq:best_prediction}
\end{equation}
the consistency of the learning algorithm is given by the convergence of the regret $r_e$, which is defined as the Euclidean distance between the predictor $\hat{g}^e$ and the optimal predictor $\hat{g}^*$:
\begin{equation}
r_e=\|\hat{g}^e - \hat{g}^*\|_2 \rightarrow 0,\ \text{as} \ e \rightarrow \infty.
\label{eq:consistency_definition}
\end{equation}
We prove the consistency of Algorithm~\ref{Algo:inf_scp} in Theorem~\ref{thm:consistency}.
\begin{theorem}\label{thm:consistency}
If Assumption~\ref{asmp:learning_bounds} holds with the maximum prediction error $\max_{x \in \mathcal{X}_e} {\|g - \mu_{g}^e\|}^{2}_2  \triangleq c_1^{e} $ at epoch $e$, then the regret $r_e = \|\hat{g}^e - \hat{g}^*\|_2$ of Algorithm~\ref{Algo:inf_scp} achieves the bound:
\begin{equation}
r_e^2 \le c_1^{e} + dC^2 +\epsilon^{*2}\textrm{ with probability $1-|\mathcal{X}_e| \delta_e$}, \label{eq:regret_bound}
\end{equation}
\begin{equation}
\textrm{where } \ \delta_{e} \triangleq  (C^{d} \prod_p^{d}\Delta_p)^{-1} \tfrac{1}{(2\pi)^{\frac{d}{2}}|\Sigma_{g}^{e}|^{\frac{1}{2}}}e^{-\tfrac{1}{2}C^2{\sum_{p=1}^{d}}\Delta_p} \label{eq:consistency_delta}
\end{equation}
with the output dimension $d$, $\Delta_p \triangleq \sum_{q=1}^d m_{qp} >0$, $\forall p \in \{1,2,\dots,d\}$, and $\mathcal{M} = {(\Sigma_{g}^{e})}^{-1}$ with its $(q,p)$-th element $(m_{qp})$. Furthermore, the regret $r_e \to 0$ as $e \to \infty$.
\end{theorem}
\begin{proof}
Using the inequality $\mathrm{Pr}(\hat{g}^e - \mu_{g}^{e}\ge C\mathbf{e})< \delta_{e}$, $\forall x \in \mathcal{X}_e$, since $\hat{g}^e \sim \mathcal{N}(\mu_{g}^{e}, \Sigma_{g}^{e})$, we can bound $r_{e}$ as follows:
\begin{equation}
    \begin{aligned}
    r^2_e \le \|\mu_{g}^{e} + C\mathbf{e} - \hat{g}^*\|_2^2,
    \end{aligned} \label{eq:theorem3_eq1}
\end{equation}
where $\delta_e$ is defined in~\eqref{eq:consistency_delta}, and $\mathbf{e}$ is unit vector in $d$ dimensions. Hence, $C$ is a function of $|\mathcal{X}_e|\delta_e$. This is the tail probabilities inequality~\cite{savage1962mills} of multivariate Gaussian distributions. The error $r_e^2$ is then bounded using the empirical prediction error $\|\mu_{g}^{e} - g\|_2^2$ and the best prediction error~\eqref{eq:best_prediction} as $r^2_e \leq  \|\mu_{g}^{e} + C\mathbf{e} - g\|_2^2 + \|g - \hat{g}^*\|_2^2$. By using the triangular inequality, we get the bound in~\eqref{eq:regret_bound} as follows: \begin{align}
    r^2_e \leq\|\mu_{g}^{e} - g\|_2^2 + dC^2 + \epsilon^{*2}
    \le &c_1^{e}  + dC^2 + \epsilon^{*2}.
\end{align}
At any epoch $e$, we have $x \in \mathcal{X}_e$, $|\|\mu_{g}^{e}\|_2^2 + \mathrm{tr}(\Sigma_{g}^{e}) - \|g\|_2^2| \le \omega_e$ ~\cite{liu2019robust}, where $\omega_e$ is a hyper-parameter that is associated with model selection in robust regression. Using $\mathrm{tr}(\Sigma_{g}^{e}) \leq c^{e}_2$,~\eqref{eq:theorem3_eq1} becomes $r^2_e \le \omega_e +  c_2^{e} + dC^2 + \epsilon^{*2}.$

 Assuming that the ground truth model is in the function class $\mathcal{G}$, we have $\epsilon^* = 0$. With infinite data, the hyper parameter $\omega_e = 0$ and $c^{e}_2 \to 0$ at each data point. For achieving the same value of $\delta_{e}$, $C \to 0$ with shrinking $\Sigma_{g}^{e}$. This implies that for large data, i.e., as $e \to \infty$, we have $r_e \to 0$. The regret $r_e$ is bounded as shown in~\eqref{eq:regret_bound}. Therefore, Algorithm~\ref{Algo:inf_scp} is consistent with high probability. Note that the Info-SNOC approach outputs trajectories of finite variance at epoch $e$, given that $\mathrm{tr}(\Sigma_{g}^{e})$ is bounded.
\end{proof}
During the initial exploration phase, the finite dimensional approximation in the gPC projection might add to the error bounds discussed in Theorems~\ref{thm:quad_chance_constraint} and~\ref{thm:lin_chance_constraint} leading to safety violation. This approximation error is automatically incorporated as residual dynamics in the stochastic dynamics when collecting data during rollout and learned using robust regression. Thus ensuring that safety is guaranteed with increasing data.
%%=============================Example================================%%
\section{Simulation and Discussion}
\label{sec:example}
% \subsubsection*{Problem Setup} 
We test the Info-SNOC framework shown in Fig.~\ref{fig:work-flow} on the three degree-of-freedom robotic spacecraft simulator~\cite{nakka2018six} dynamics with an unknown friction model and an over-actuated thruster configuration that is used for a real spacecraft. The dynamics of $\vec{x}=(x,y,\theta)^T$ with respect to an inertial frame is given as: 
\begin{equation}
    \begin{aligned}
    \ddot{\vec{x}} = f(\theta,H,u) + g(\dot{x},\dot{y},\dot{\theta}),\ f =  \mathrm{blkdiag}(R(\theta),1)H u,
    \end{aligned}\nonumber
\end{equation}
where $R(\theta) \in SO(2)$. The states $(x,y) \in \mathbb{R}^{2}$, $\theta \in [0,2\pi)$ denote position and orientation respectively. The function $g$ is unknown, and assumed to be linear viscous damping in the simulations. The control effort $u \in \mathbb{R}^{8}$ is constrained to be $0 \leq u \leq 1$, and $H(m,I,l,b) \in \mathbb{R}^{3 \times 8}$ is the control allocation matrix where $m = $\SI{17}{kg} and $I =$ \SI{2}{kg.m.s^{-2}} are the mass and the inertia matrix, and $l=b=$ \SI{0.4}{m} is the moment arm.
\begin{figure}
\vspace{5pt}
\centering{
\subfloat[]{
    \includegraphics[width=0.22\textwidth,height=0.9in]{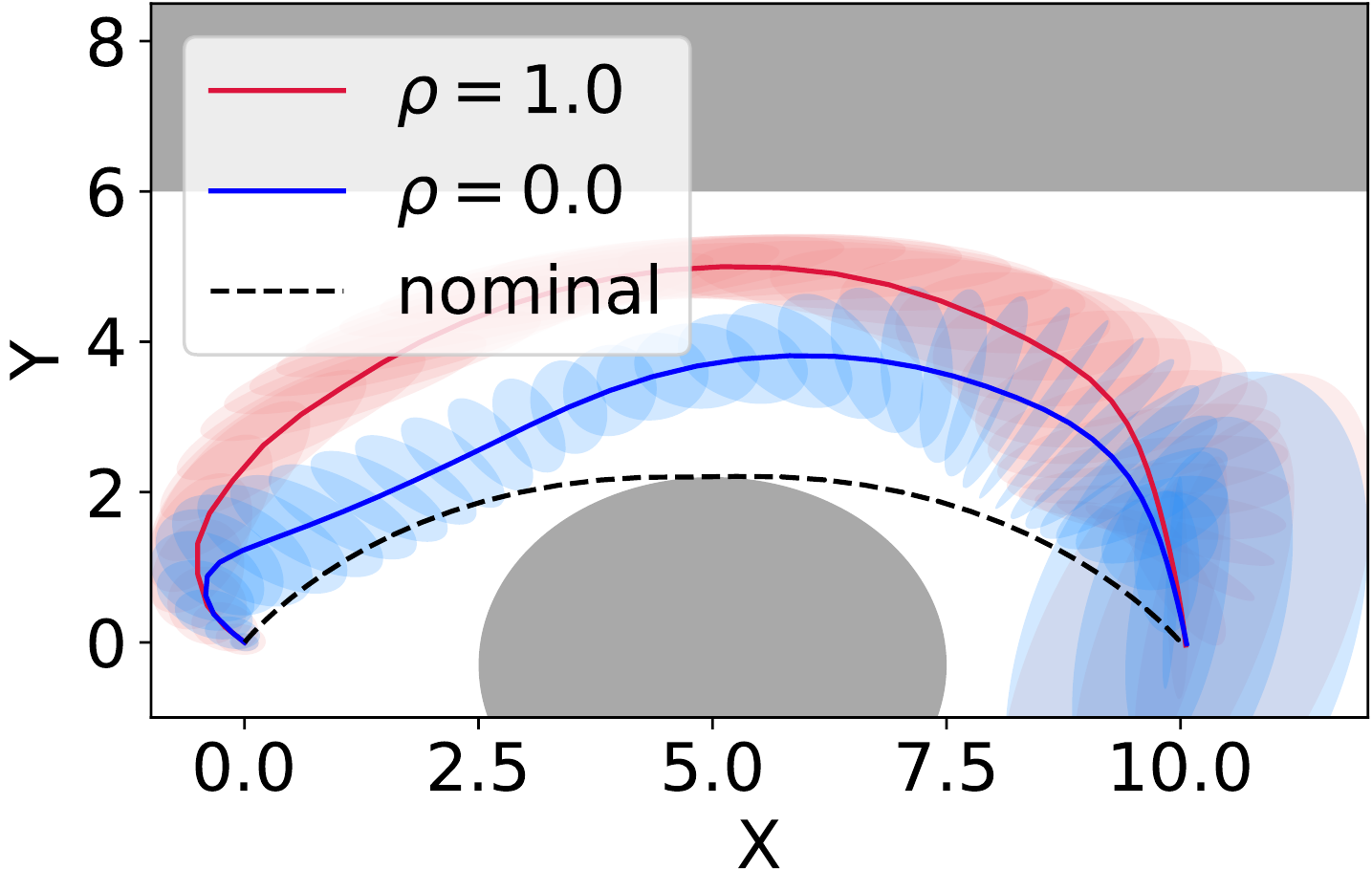}
    \label{fig:case1_perf_inf}
    }
}
\centering{
\subfloat[]{
    \includegraphics[width=0.22\textwidth,height=0.9in]{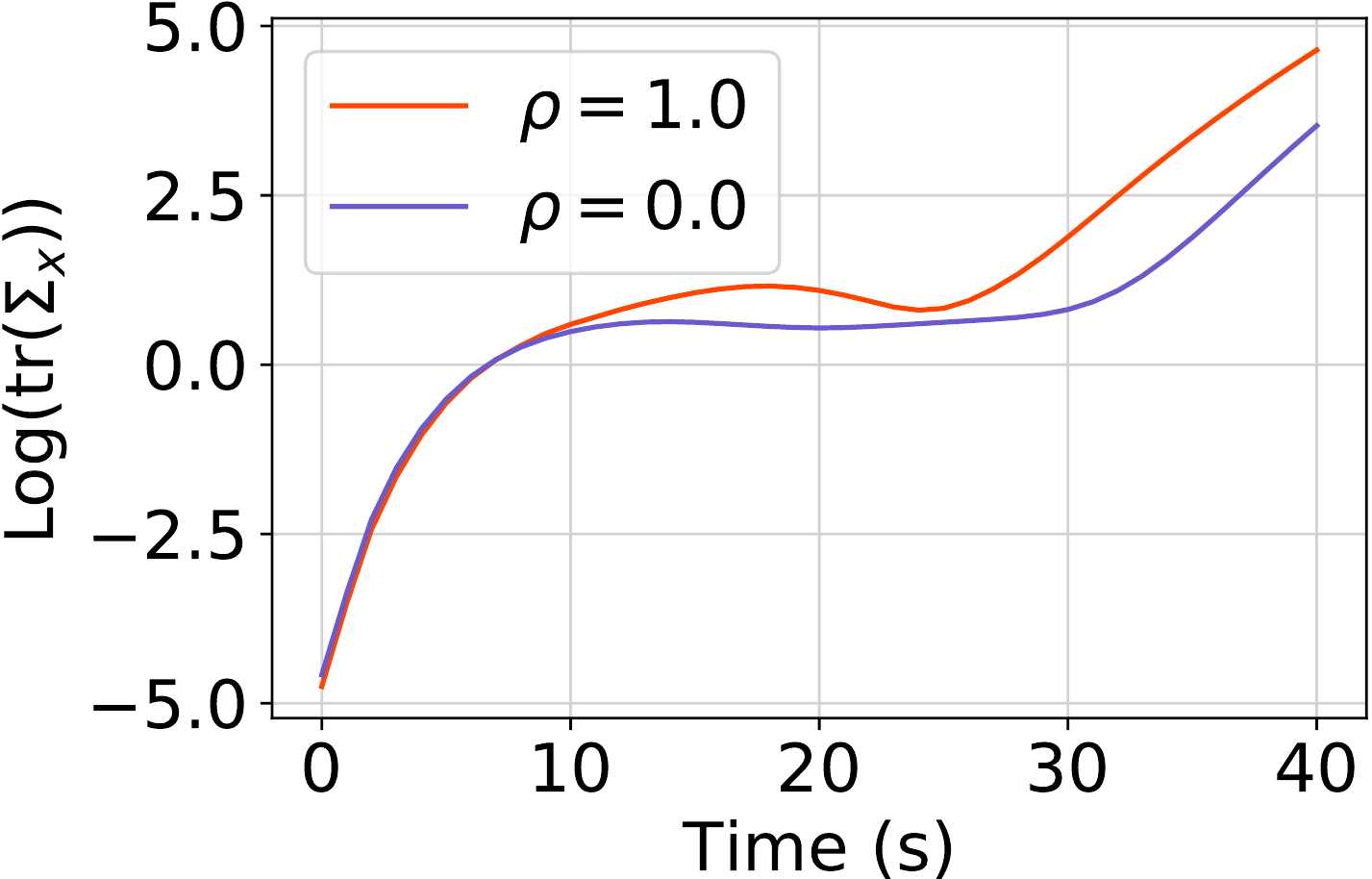} 
    \label{fig:case1_tr_time}
    }
}
\centering{
\subfloat[]{
    \includegraphics[width=0.22\textwidth,height=0.9in]{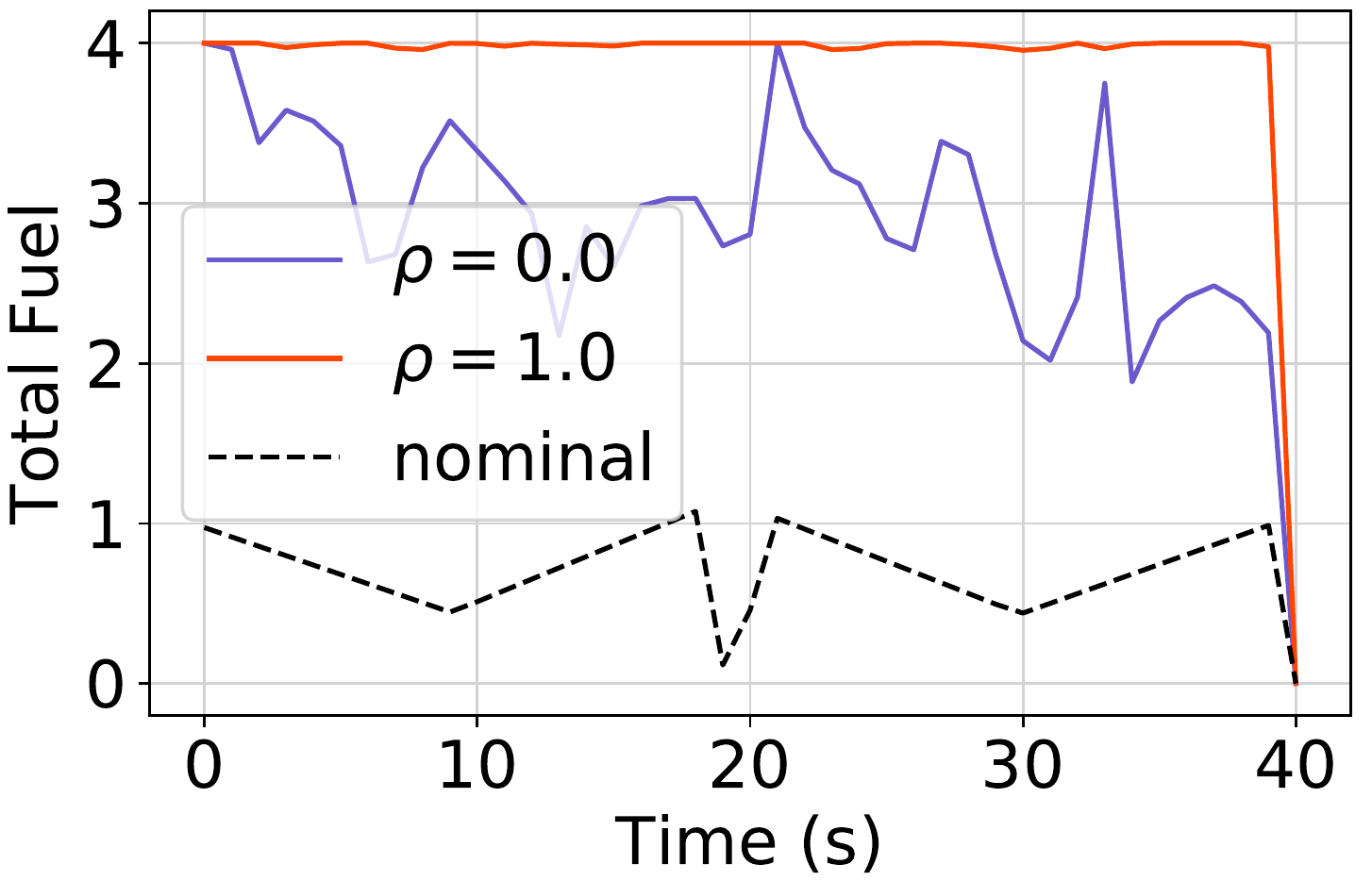}
    \label{fig:case1_fuel_time}
    }
}
\centering{
\subfloat[]{
    \includegraphics[width=0.22\textwidth,height=0.9in]{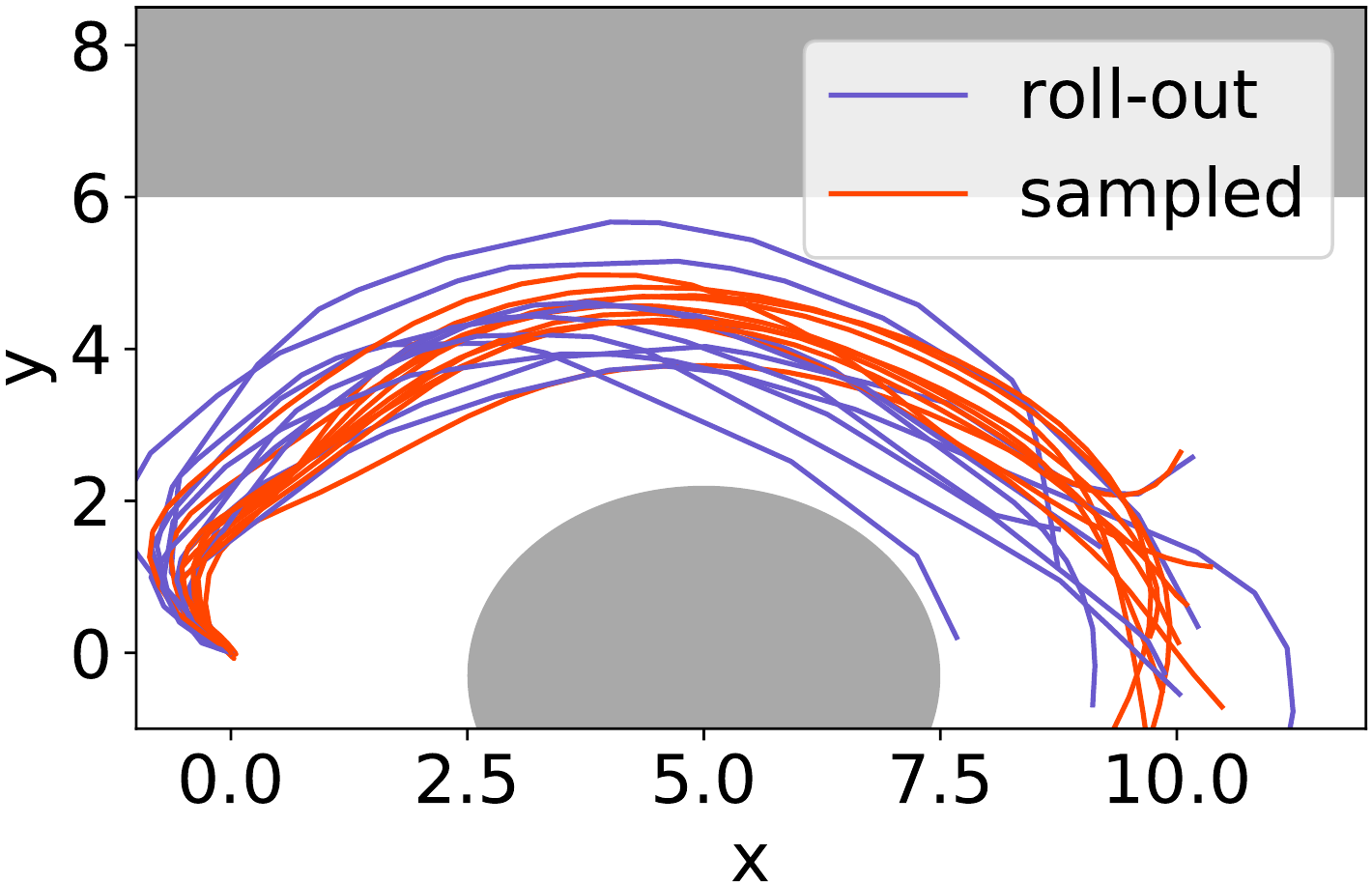}
    \label{fig:case1_roll_out}
    }
}
\caption{Info-SNOC applied to Scenario 1. In Fig.~(a), we show the motion plan along with the 2$\sigma$ confidence in position of the performance trajectory ($\rho = 0$) and  the information trajectory ($\rho=1$) computed using Info-SNOC, and the nominal trajectory computed using SCP under nominal dynamics. In Fig.~(b), we show the trace of $\Sigma_x$ w.r.t time. The information trajectory ($\rho=1$) has higher $\Sigma_x$ compared to the performance trajectory ($\rho = 0$). We compare total open-loop fuel computed at each time step in Fig.~(c), and in Fig.~(d) we demonstrate collision avoidance during exploration for 20 trials of rollout using a safety augmented stable controller.}
\label{fig:case1}
\vspace{-10pt}
\end{figure}
The unknown function $g = \mathrm{diag}([-0.02, -0.02, -0.002])\dot{x}$ is modeled as a multivariate Gaussian distribution to learn from data using robust regression. To get an initial estimate of the model, we explore a small safe set around the initial condition and collect data. For the following test cases, we collect 40 data points to have a feasible optimal control problem~\eqref{eq:stoptprob_sicc} in the planning stage as discussed in Sec.~\ref{sec:computation}. The algorithm is initialized with a \emph{nominal trajectory} computed using deterministic SCP under nominal dynamics.

%---------------Info-SNOC Results---------------------------
\subsubsection*{Info-SNOC Results} The learned dynamics is used to design safe trajectories for \SI{40}{s} using the Info-SNOC algorithm for Scenarios 1 and 2 as shown in Figs.~\ref{fig:case1_perf_inf} and~\ref{fig:case2_perf_inf}, respectively. \textcolor{black}{Scenario 1 has a wall at $y=$ \SI{6}{m} and a circular obstacle of radius \SI{2.5}{m} at $(5,-0.3)$, and scenario 2 has two circular obstacles of radius \SI{0.6}{m} at $(5,-0.5)$ and $(10,0.5)$ as collision constraints respectively.} The obstacles in both scenarios are transformed to linear chance constraints and the terminal constraint is transformed to quadratic chance constraint as discussed in~\cite{nakka2019nsoc}, with a risk measure of $\epsilon_{\ell} = \epsilon_{q}= 0.05$. The Info-SNOC algorithm is applied with $\rho = 0$ ($\mathcal{L}_{1}$-norm control cost) and $\rho=1$ (information cost). We compare the mean $\mu_x$ and $2\sigma$-confidence ellipse around $\mu_x$ of the trajectories with the nominal trajectory. We observe that for the $\rho=1$, the safe trajectory explores larger state-space compared to the $\rho=0$ case, which corresponds to the fuel-optimal trajectory. The total control effort at each time is shown in Fig.~\ref{fig:case1_fuel_time}. It shows that information trajectory uses more energy compared to the performance trajectory. The extra fuel is used to explore the domain for improving the model. The terminal variance $\Sigma_x(t_f)$ in both scenarios is large due to the correlation among the multiple dimensions of $g$ that are predicted by the learning algorithm.
\begin{figure}
\vspace{5pt}
\centering{
\subfloat[]{
    \includegraphics[width=0.22\textwidth,height=0.9in]{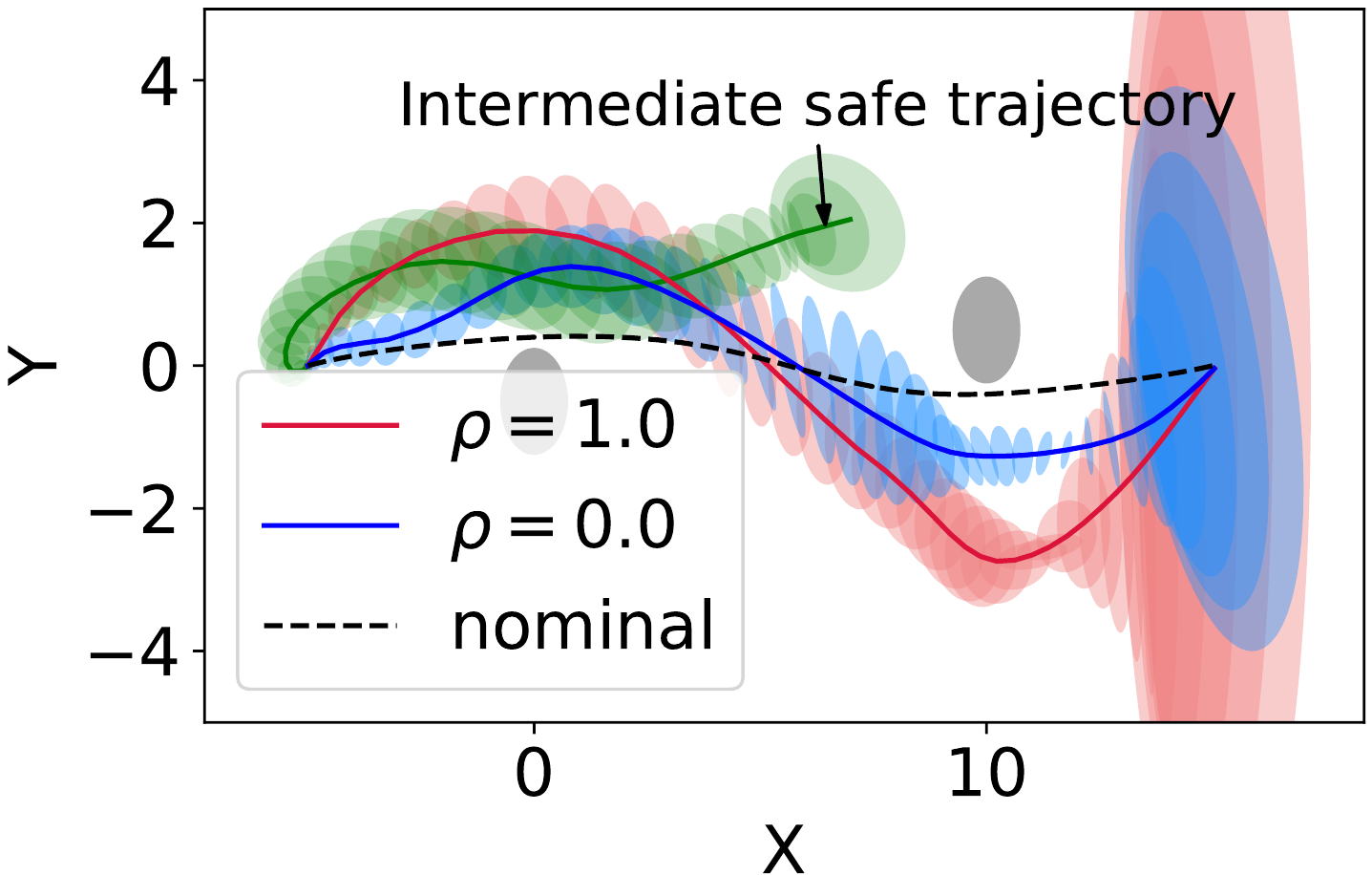}
    \label{fig:case2_perf_inf}
    }
}
\centering{
\subfloat[]{
    \includegraphics[width=0.22\textwidth,height=0.9in]{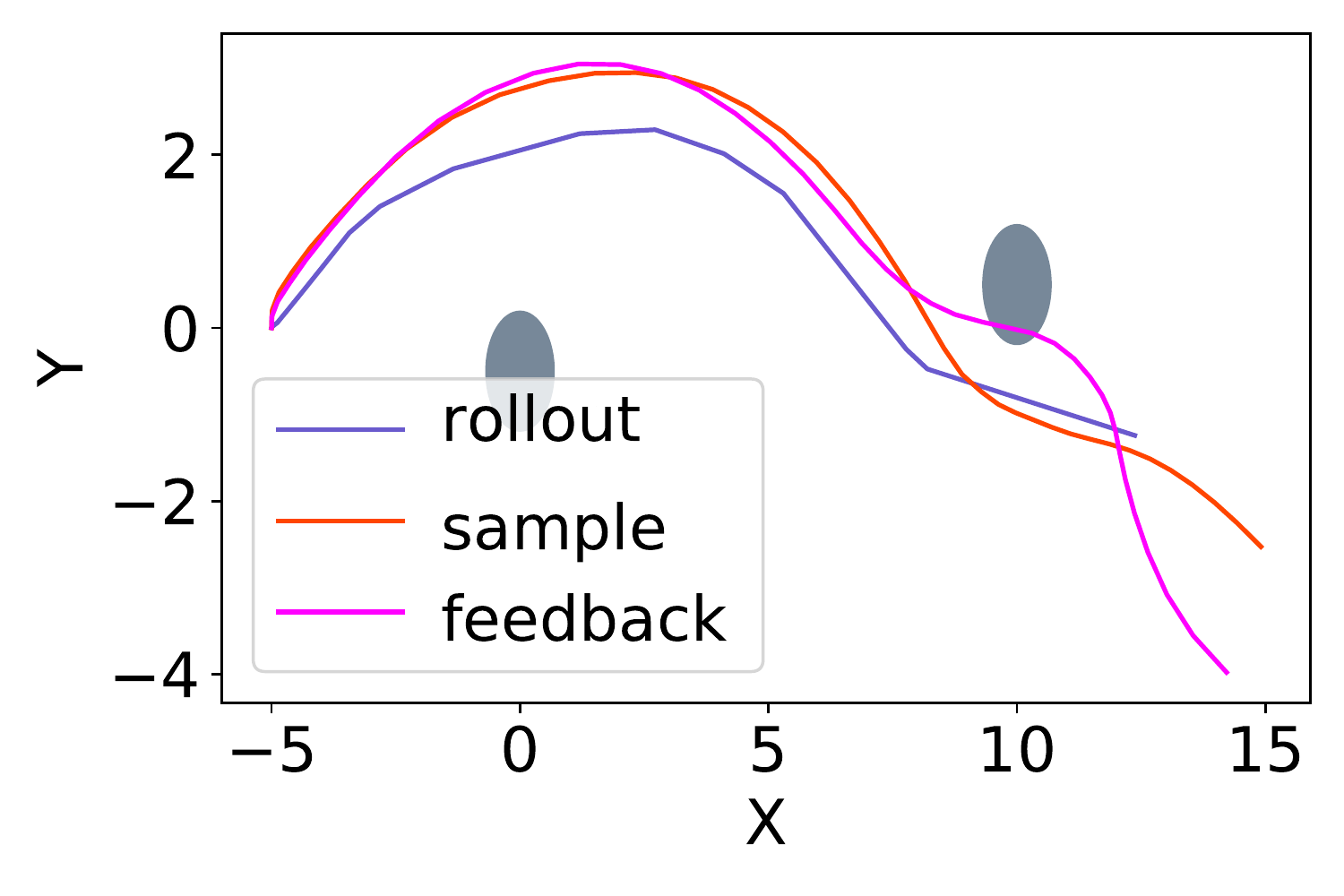}
    \label{fig:case2_roll_out}
    }
}
\caption{\textcolor{black}{Info-SNOC applied to Scenario 2. In Fig.~(a), we show a comparison of the performance trajectory ($\rho = 0$), the information trajectory ($\rho=1$), and an intermediate safe trajectory (green) computed using Info-SNOC and the nominal trajectory computed using deterministic SCP under nominal dynamics. In Fig.~(b), we compare a sampled trajectory with the trajectories generated by feedback tracking and rollout with a safety filter.}}
\label{fig:case2}
\vspace{-10pt}
\end{figure}
\subsubsection*{Safe Rollout} The safe trajectories in Figs.~\ref{fig:case1_perf_inf} and~\ref{fig:case2_perf_inf} are sampled following the method discussed in Sec.~\ref{subsec:exploration} and the rollout is performed using the controller designed in~\cite{nakka2018six} that satisfies Lemma~\ref{lemma:bound}. The sample trajectories and rollout trajectories are shown in red and blue respectively in Figs.~\ref{fig:case1_roll_out} and~\ref{fig:case2_roll_out}. The sampled trajectory is safe with the risk measure of collision $\epsilon_{\ell} = 0.05$ around the obstacles. Rollout trajectories, with a feedback controller, collide with the obstacles due to the following two reasons: 1) the learning bounds~\eqref{eq:learning_bounds} lead to constraint violation as discussed in Theorems~\ref{thm:quad_chance_constraint} and~\ref{thm:lin_chance_constraint}, and 2) \textcolor{black}{the state-dependent uncertainty model $\mathcal{N}(\mu_g,\Sigma_g)$ might predict large $\Sigma_g$ that can saturate the actuators.} Saturated actuators cannot compensate for the unmodelled dynamics.
In order to ensure safety, we augment the feedback controller with a real-time safety augmentation using barrier-function -based quadratic program~\cite{ames2019control}. Using this filter, the blue rollout trajectories are diverted from obstacles, as seen in Figs.~\ref{fig:case1_roll_out} and~\ref{fig:case2_roll_out}, avoiding constraint violation.
\begin{figure}[h]
    \centering
    \includegraphics[width=3.4in,height=0.9in]{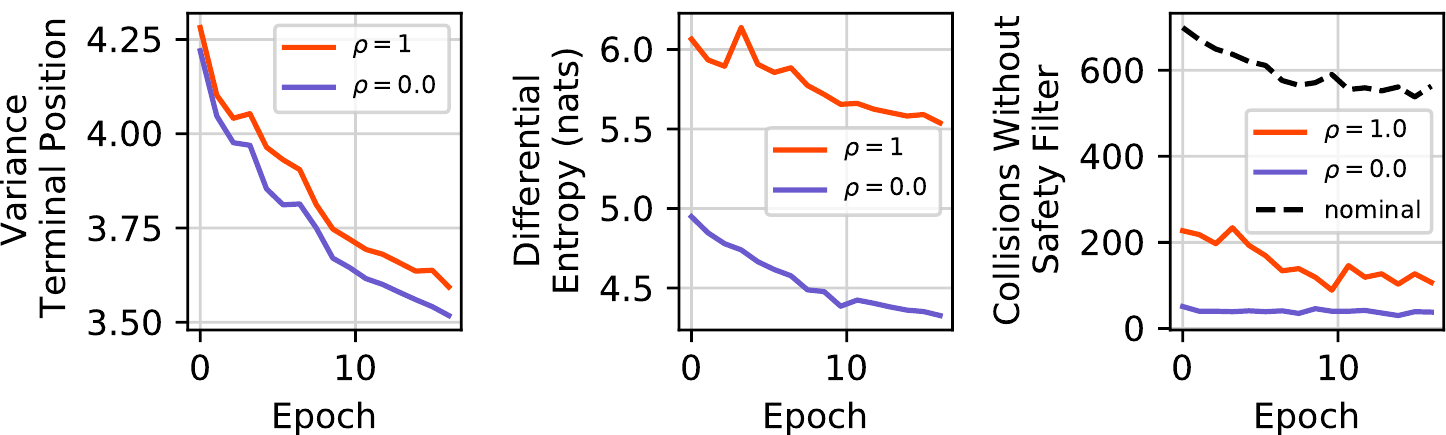}
    \caption{\textcolor{black}{Performance over epochs for Scenario 1. Left: we show decrease in the terminal position variance over epochs demonstrating improved goal reaching with epoch. Center: the differential entropy of the prediction variance $\Sigma_{g}$ for information trajectory ($\rho=1$) is larger compared to the performance trajectory ($\rho=0$). Right: the number of collisions during rollout for $1000$ trials decrease as the learning converges, validating Theorem~\ref{thm:lin_chance_constraint}.}}
    \label{fig:learning_performance}
\end{figure}
%--------------------------- Consistency --------------------------------------
\subsubsection*{Consistency} The data collected during rollout is appended to the earlier data to learn a new model. Figure~\ref{fig:learning_performance} \textcolor{black}{shows improvement in control performance (control cost for $\rho = 0.0$) with increasing number of epochs. The prediction variance $\Sigma_{g}$ decreases from $1e-2$ to $0.3e-2$ over 15 epochs.} The main assumption for consistency, which states that the variance of the trajectory computed using~\eqref{eq:stoptprob_sicc} is bounded and decreasing, is satisfied, thereby demonstrating the correctness of Theorem~\ref{thm:consistency}. We observe that the differential entropy of the information trajectory is higher (i.e., contains more information about $g$) than the performance trajectory by design.

We observed that $\Sigma_{x}(t_f)$ of the information trajectory $(\rho=1)$ computed using Info-SNOC decreases from $104.11$ to $61.46$ over 15 epochs, by applying the framework in Fig.~\ref{fig:work-flow} for Scenario 1, demonstrating the increase in probability of reaching the quadratic terminal set. This validates the Theorem~\ref{thm:quad_chance_constraint}. The number of collisions of the rollout trajectory without a safety filter over a $1000$ trials at each epoch is shown in Fig.~\ref{fig:learning_performance}. The number of collisions decrease from $227$ to $107$ for the information trajectory, from $51$ to $38$ for the performance trajectory, and from $699$ to $563$ for the nominal trajectory over $15$ epochs. This increase in the probability of the linear chance constraint satisfaction with epoch validates the Theorem~\ref{thm:lin_chance_constraint}. In the first epoch, we observed that the rollout using the nominal trajectory leads to $69.9\%$ collisions over $1000$ trials, using the information trajectory leads to $22\%$ collisions, and using the performance trajectory leads to less than $5.1\%$ collisions, demonstrating the effectiveness of Info-SNOC.
\section{Conclusion} \label{sec:conclusion}
We present a new method of learning-based safe exploration and motion planning by solving information-cost stochastic optimal control using a partially learned nonlinear dynamical model. The variance prediction of the learned model is used as the information cost, while the safety is formulated as distributionally robust chance constraints. The problem is then projected to the generalized polynomial chaos space and solved using sequential convex programming. We use the Info-SNOC method to compute a safe and information-rich pool of trajectories for rollout using an exponentially stable controller with a safety filter augmentation for safe data collection. We analyze the probability of constraint violation for both linear and quadratic constraints. We show that the safety constraints are satisfied for rollout under learned dynamics, as the learned model converges to the optimal predictor over epochs. The consistency of the learning method using the Info-SNOC algorithm is proven under mild assumptions. 

The episodic learning framework was applied to the robotic spacecraft model to explore the state space and learn the friction under collision constraints. We compute a pool of safe and optimal trajectories using the Info-SNOC algorithm for a learned spacecraft model under collision constraints and discuss an approach for rollout using a stable feedback control law to collect data for learning. We validate the consistency of robust regression method and the safety guarantees by showing the reduction in variance of the learned model predictions and the number of collisions over 15 epochs respectively.
\vspace{-7pt}
\bibliographystyle{IEEEtran}
\bibliography{learn_soc.bib}
\end{document}